\def\year{2021}\relax
\pdfoutput=1
\documentclass[letterpaper]{article}
\usepackage{aaai21}
\usepackage{times}
\usepackage{helvet}
\usepackage{courier}
\usepackage{algorithm}
\usepackage{algorithmicx}
\usepackage{algorithmicx}
\usepackage{algpseudocode}
\usepackage[hyphens]{url}
\usepackage{amsmath}
\usepackage{booktabs}
\usepackage{float}
\usepackage{amssymb}
\usepackage{amsthm}
\usepackage{subfigure}
\usepackage{diagbox}
\usepackage{threeparttable}
\usepackage{multirow}
\usepackage{natbib}
\usepackage{caption}

\usepackage{graphicx}
\usepackage{epstopdf}
\graphicspath{{../eps/}}
\DeclareGraphicsExtensions{.eps}
\usepackage{tikz} % the tikz package
\usetikzlibrary{shapes.geometric, arrows}

\title{Going Deeper With Directly-Trained Larger Spiking Neural Networks}
\author {
        Hanle Zheng \textsuperscript{\rm 1}, Yujie Wu \textsuperscript{\rm 1}, Lei Deng \textsuperscript{\rm 1,3}, Yifan Hu \textsuperscript{\rm 1} and Guoqi Li \textsuperscript{\rm 1,2 }\thanks{Corresponding author: Guoqi Li (liguoqi@tsinghua.edu.cn)}\\
}
\affiliations {
    \textsuperscript{\rm 1} Center for Brain-Inspired Computing Research, Department of Precision Instrument, Tsinghua University, Beijing 100084, China \\
    \textsuperscript{\rm 2} Beijing Innovation Center for Future Chip, Tsinghua University, Beijing 100084, China \\
    \textsuperscript{\rm 3} Department of Electrical and Computer Engineering, University of California, Santa Barbara \\

}
\begin{document}
\maketitle

\begin{abstract}
Spiking neural networks (SNNs) are promising in a bio-plausible coding for spatio-temporal information and event-driven signal processing, which is very suited for energy-efficient implementation in neuromorphic hardware. However, the unique working mode of SNNs makes them more difficult to train than traditional networks. Currently, there are two main routes to explore the training of deep SNNs with high performance. The first is to convert a pre-trained ANN model to its SNN version, which usually requires a long coding window for convergence and cannot exploit the spatio-temporal features during training for solving temporal tasks. The other is to directly train SNNs in the spatio-temporal domain. But due to the binary spike activity of the firing function and the problem of gradient vanishing or explosion, current methods are restricted to shallow architectures and thereby difficult in harnessing large-scale datasets (e.g. ImageNet). To this end, we propose a threshold-dependent batch normalization (tdBN) method based on the emerging spatio-temporal backpropagation, termed ``STBP-tdBN", enabling direct training of a very deep SNN and the efficient implementation of its inference on neuromorphic hardware. With the proposed method and elaborated shortcut connection, we significantly extend directly-trained SNNs from a shallow structure ($<$10 layer) to a very deep structure (50 layers). Furthermore, we theoretically analyze the effectiveness of our method based on ``Block Dynamical Isometry" theory. Finally, we report superior accuracy results including 93.15\% on CIFAR-10, 67.8\% on DVS-CIFAR10, and 67.05\% on ImageNet with very few timesteps. To our best knowledge, it's the first time to explore the directly-trained deep SNNs with high performance on ImageNet. We believe this work shall pave the way of fully exploiting the advantages of SNNs and attract more researchers to contribute in this field.

\end{abstract}

\section{Introduction}
Inspired by human neurons' working patterns, spiking neural networks (SNNs) have been considered as a promising model in artificial intelligence and theoretical neuroscience \cite{PMID:31776490}. Benefited from intrinsic neuronal dynamics and event-driven spike communication paradigms, SNNs show great potential in continuous spatio-temporal information processing with lower energy consumption and better robustness\cite{stromatias2015robustness}. Moreover, SNNs can be easily applied on some specialized neuromorphic hardware \cite{merolla2014million,davies2018loihi}, which may be seen as the next generation of neural networks.

  There are two main approaches to train an SNN with high performance. One is to convert a pre-trained ANN to an SNN model, which usually needs hundreds of timesteps\cite{sengupta2019going,hu2018spiking}. So, even though these SNNs achieve comparable accuracy to ANNs with similar structures, the large number of timesteps causes serious signal latency and increase the amount of computation. The other is to train SNNs directly based on gradient descent method, which is independent of pre-trained ANNs and can lessen the number of timesteps. A recent work \cite{wu2018spatio}, which proposes a learning algorithm called ``Spatio-temporal backpropagation" (STBP) with an approach to training SNNs directly on ANN-oriented programming frameworks (e.g. Pytorch), provides us with a chance to explore deeper and larger directly-trained SNNs. However, SNNs trained by this algorithm are restricted to shallow architectures and cannot achieve satisfactory performance on large-scale datasets such as ImageNet. So under above algorithm, we clarify two problems to be solved for training deeper SNNs directly.

  The first problem is gradient vanishing or explosion. Because of special mechanism of spatio-temporal information processing and non-differentiable spiking signal, the gradient propagation behaves much unstable and tends to vanish in most cases when training SNNs directly, which prevents SNNs from going deeper. So far, there isn't an effective method to handle this problem well in directly-trained SNNs. Famous works \cite{lee2020enabling,lee2016training,wu2019direct} before us fail to train deep SNNs directly and all of their models are less than 10 layers, which seriously influences the performances of their methods.

   The other problem is that we need to balance the threshold and input on each neuron to get appropriate firing rates in SNNs. When the input is too small compared with the threshold, the neuron fires few spikes and the neuronal membrane potential remains so the information handled by the neuron cannot be expressed enough. When the input is too large, the neuron fires all the time and is insensitive to the change of input. For directly-trained SNNs, with binary spikes propagating layer by layer, the distribution of pre-synaptic inputs will shift during the training process, making the size of inputs inappropriate. Many methods have been proposed to deal with it, such as threshold regularization \cite{lee2016training} and NeuNorm \cite{wu2019direct}.

   Normalization seems to be appropriate methods to solve both of the problems. They stabilize the network and gradient propagation according to \cite{chen2020comprehensive}. Also, they normalize the distributions of pre-synaptic inputs to same expectation and variance, which helps to balance the threshold and input by reducing the internal covariate shift. However, the existing normalization methods aren't suitable for training of SNNs. For the additional temporal dimension and special activation mechanism, directly-trained SNNs need a specially designed normalization method .

  In this paper, we develop a new algorithm to train a deep SNN directly. The main contributions of this work are summarized as follows:
  \begin{itemize}
  \item We propose the threshold-dependent batch normalization to solve the gradient vanishing or explosion problem and adjust firing rate. Furthermore, we take the residual network structure and modify the shortcut connections which are suitable for deep SNNs.
  \item On this basis, we investigate very deep directly-trained SNNs (\textbf{extending them from less than 10 to 50 layers}) and test them over large-scale non-spiking datasets (CIFAR-10, ImageNet) and neuromorphic datasets (DVS-Gesture, DVS-CIFAR10).
  \item On CIFAR-10 and ImageNet,  we comprehensively validate different SNN architectures (ResNet-18, 34, 50) and report competitive results compared to similar SNNs with much fewer timesteps (\textbf{no more than 6 timesteps}), to our best knowledge, which is the first time that the directly-trained SNN with full spikes report fairly high accuracy on ImageNet. On neuromorphic datasets, our model achieves state-of-the-art performance on both DVS-Gesture and DVS-CIFAR10, which shows advantage of SNNs on dealing with temporal-spatial information.
  \end{itemize}
\section{Related Work}
\paragraph{Learning algorithm of SNNs}In the past few years, a lot of learning algorithms have explored how to train a deep SNN with high performance, including: (1) some  approaches to converting pre-trained ANNs to SNNs; (2) gradient descent based algorithms.

  The first one is called as the ``ANN-SNN conversion methods" \cite{sengupta2019going,han2020rmp}, seen as the most popular way to train deep SNNs with high performance, which transforms the real-valued output of ReLU function to binary spikes in the SNN. This kind of method successfully reports competitive results over large-scale datasets without serious degradation compared to ANNs. However, it ignores rich temporal dynamic behaviors of spiking neurons and usually requires hundreds or thousands of time steps to approach the accuracy of pre-trained ANNs.

  The gradient descent based algorithms train SNNs with error backpropagation. With gradient descent optimization learning algorithms, some SNN models \cite{lee2016training,jin2018hybrid,lee2020enabling} achieve high performance on CIFAR-10 and other neuromorphic datasets. Among them, \cite{wu2019direct} improves the leaky integrate-and-fire (LIF) model \cite{hunsberger2015spiking} to an iterative LIF model and develop STBP learning algorithm, which makes it friendly on ANN-oriented programming frameworks and speeds up the training process. Moreover, training SNNs directly shows great potential on dealing with spatial and temporal information and reports high accuracy within very few timesteps. However, it fails to train a very deep SNN directly because of the gradient vanishing and internal covariate shift, which is exactly what we want to conquer.
\paragraph{Gradient vanishing or explosion in the deep neural network (DNN)}
 A DNN can avoid gradient vanishing or explosion when it is dynamic isometry, which means every singular value of its input-output jacobian matrix remains close to 1. \cite{chen2020comprehensive} proposes a metric---``\textbf{Block Dynamical Isometry}", serving as a general statistical tool to all of complex serial-parallel DNN. It investigates the first and second moment of each block in the neural network and analyze their effects to the gradient distribution. Moreover, it gives theoretical explanation to the function of the weight initialization, batch normalization and shortcut connection in the neural network, which helps us to develop our algorithm.
 \paragraph{Normalization}
 Normalization techniques enable the training of well-behaved neural networks. For artificial neural networks, normalization, such as batch normalization \cite{ioffe2015batchnorm}, group normalization \cite{wu2018groupnorm}, and layer normalization \cite{layernorm}, have become common methods. Batch normalization (BN) accelerates deep networks training by reducing internal covariate shift, which enables higher learning rates and regularizes the model. While it causes high learning latency and increases computation, BN makes it possible for networks to go deeper avoiding gradient vanishing or explosion. For SNNs, researchers propose other normalization techniques, such as data-based normalization \cite{diehl2015fast}, Spike-Norm \cite{sengupta2019going} and NeuNorm \cite{wu2019direct}. These normalization methods aim to balance the input and threshold to avoid serious information loss, but they are not effective to our directly-trained deep SNNs because they still neglect the problem of gradient vanishing. We noticed the effects of BN in ANNs and the importance of input distribution in SNNs, so we modify BN to satisfy the training and inference of SNN models.
 \section{Materials and Methods}
 \subsection{Iterative LIF model}
   The iterative LIF model is first proposed by \cite{wu2019direct}, who utilizes Euler method to solve the first-order differential equation of Leaky integrate-and-fire (LIF) model and converts it to an iterative expression
   \begin{equation} u^{t} = \tau_{decay}u^{t-1}+I^{t},\end{equation}
   where $\tau_{decay}$ is a constant to describe how fast the membrane potential decays, $u^t$ is the membrane potential, $I^t$ is the pre-synaptic inputs. Let $V_{th}$ denote the given threshold. When $u^t>V_{th}$, the neuron fires a spike and $u^t$ will be reset to 0. The pre-synaptic inputs are accumulated spikes from other neurons at the last layer. So $I^{t}$ can be represented by $x^t = \sum_{j}w_jo^{t}(j)$, where $w_j$ are weights and $o^{t}(j)$ denotes binary spiking output from others at the moment of $t$.
   Taking spatial structure into consideration and set $u_{reset}=0$, the whole iterative LIF model in both spatial and temporal domain can be determined by
   \begin{align}   u^{t,n+1} &= \tau_{decay}u^{t-1,n+1}(1-o^{t-1,n+1})+x^{t,n},\\
      o^{t,n+1} &= \left\{ \begin{array}{ll} 1 & \textrm{if $u^{t,n+1}>V_{th}$},\\ 0 & \textrm{otherwise}. \end{array} \right. \end{align}
   where $u^{t,n}$ is the membrane potential of the neuron in $n$-th layer at time $t$, $o^{t,n}$ is the binary spike and $\tau_{decay}$ is the potential decay constant.

   The iterative LIF model enables forward and backward propagation to be implemented on both spatial and temporal dimensions, which makes it friendly to general machine-learning programming frameworks.
  \subsection{Threshold-dependent batch normalization}
  As a regular component of DNNs, batch normalization (BN) has been a common method for current neural networks, which allows stable convergence and much deeper neural networks. However, because of the additional temporal dimension and special activation mechanism, directly-trained SNNs need a specially-designed normalization method. That motivates us to propose the threshold-dependent batch normalization (tdBN).

  We consider a Spiking Convolution Neural Network (SCNN). Let $o^{t}$ represent spiking outputs of all neurons in a layer at timestep $t$. With convolution kernel $W$ and bias $B$, we have
  \begin{equation} x^{t} = W \circledast  o^{t} +B, \end{equation}
  where $x^{t} \in R^{N\times C\times H \times W}$ represents the pre-synapse inputs at timestep $t$ with $N$ as the batch axis, $C$ as the channel axis, $(H,W)$ as the spatial axis.

  In our tdBN, the high-dimensional pre-synaptic inputs will be normalized along the channel dimension (Fig. 1). Let $x^t_{k}$ represent $k$-th channel feature maps of $x^t$. Then $x_k = (x^1_k,x^2_k,\cdots,x^T_k)$ will be normalized by
  \begin{align} \hat x_k &= \frac{\alpha V_{th}(x_k-E[x_k])}{\sqrt{Var[x_k]+\epsilon}},\\
   y_k &= \lambda_k \hat x_k+\beta_k,  \end{align}
  where $V_{th}$ denotes the threshold, $\alpha$ is a hyper-parameter depending on network structure, $\epsilon$ is a tiny constant, $\lambda_k$ and $\beta_k$ are two trainable parameters, $E[x_k], Var[x_k]$ are the mean and variance of $x_k$ statistically estimated over the Mini-Batch. Fig. 1 displays how $ E[x_k],Var[x_k]$ to be computed, which are defined by
  \begin{align} E[x_k] &= mean(x_k),\\
   Var[x_k] &= mean((x_k-E[x_k])^2).\end{align}
  So, during the training, $y_k \in R^{T\times N\times H \times W}$ is exactly the normalized pre-synaptic inputs received by neurons of $k$-th channel at the next layer during $T$ timesteps.

\begin{figure}[!htp]
  \centering
  \includegraphics[width=0.5\textwidth]{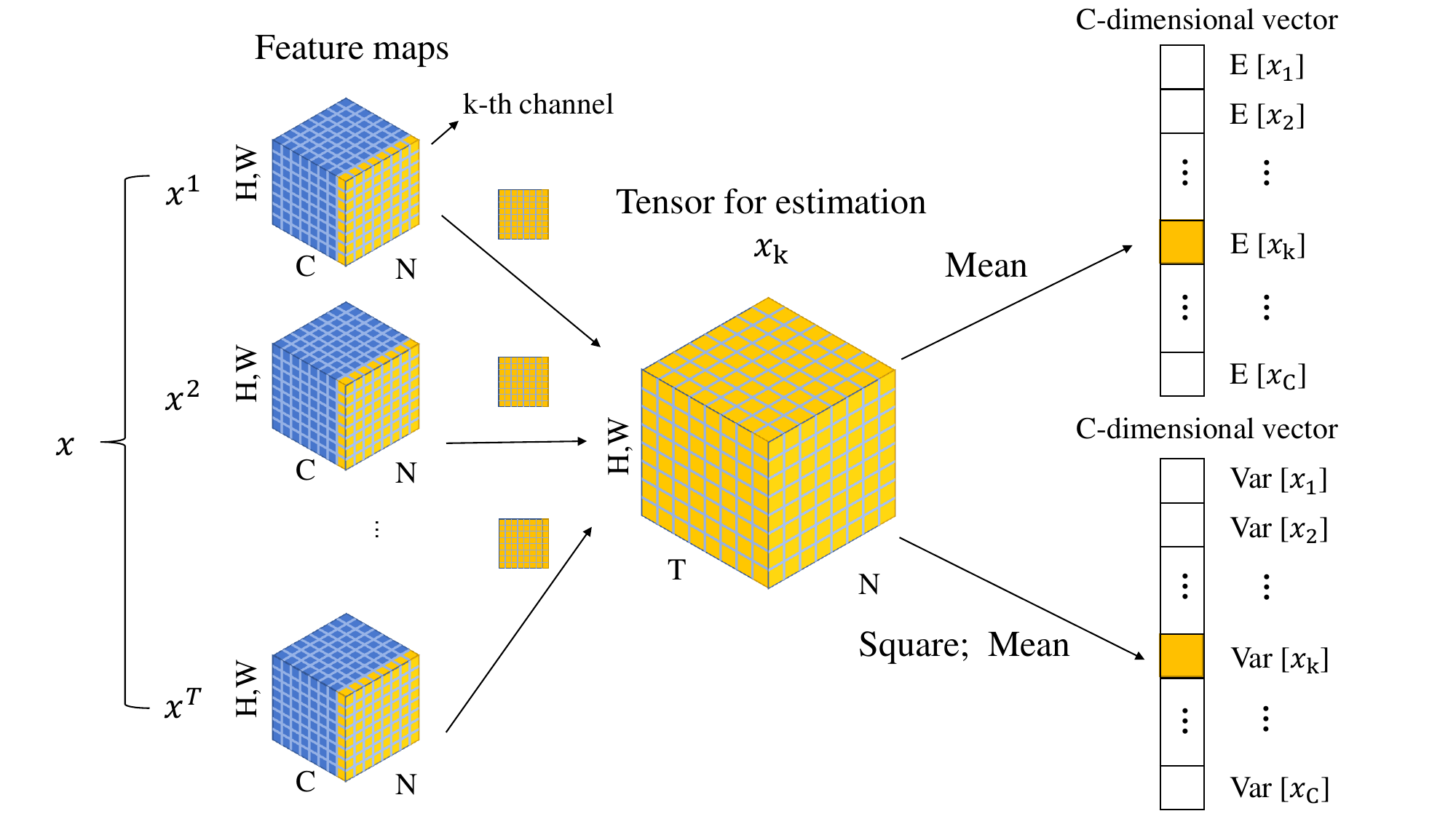}
  \caption{ \textbf{Estimation of $E[x]$ and $Var[x]$ in tdBN}. Each cube shows a feature map tensor at $t$ timestep, with $N$ as the batch axis, $C$ as the channel axis, $(H,W)$ as the spatial axis. Each element in C-dimension vector $E[x]$ and $Var[x]$ is estimated by the yellow tensor of corresponding channel.}
  \end{figure}

  In the inference, we follow the schema as standard batch normalization to estimate $\mu_{inf}$ and $\sigma^2_{inf}$ that represent the expectation of $E[x_k]$ and $Var[x_k]$ respectively over the whole datasets, which can be computed during the training process by moving average solution.

  Moreover, the batchnorm-scale-fusion is necessary to SNNs with tdBN in inference. It removes the batch normalization operations during the inference, thereby maintaining the network to be full-spiking and enabling it to be implemented on the neuromorphic platforms. Let $W_{c,k}$ and $B_{c,k}$ denote the convolution kernel and bias between the $c$-th feature map in a layer and the $k$-th feature map in the next layer. The schema is determined by

  \begin{align} W_{c,k}'&=\lambda_k\frac{\alpha V_{th}W_{c,k}}{\sqrt{\sigma^2_{inf,k}+\epsilon}},\\
   B_{c,k}'&=\lambda_k\frac{\alpha V_{th}(B_{c,k}-\mu_{inf,k})}{\sqrt{\sigma^2_{inf,k}+\epsilon}}+\beta_k,\end{align}
  where $W_{c,k}'$ and bias $B_{c,k}'$ denote the transformed weights after the batchnorm-scale-fusion. Thus, during the inference, spikes propagate layer by layer through transformed weights $W_{c,k}'$ and bias $B_{c,k}'$ without batchnorm operations. Therefore, our tdBN only affects the computation costs during training and doesn't influence the running mechanism of SNNs already trained.

  In short, our tdBN has two main differences from the standard BN. Firstly, unlike ANNs, SNNs propagate information not only layer by layer but also from last moment to the next. So, tdBN should normalize feature inputs on both temporal and spatial dimensions. Secondly, we make normalized variance dependent on threshold. In tdBN, pre-activations are normalized to $N(0,(\alpha V_{th})^2)$ instead of $N(0,1)$. And we will initialize the trainable parameters $\lambda$ and $\beta$ with 1 and 0. In the serial neural network, the hyper-parameter $\alpha$ is 1. For a local parallel network structure having $n$ branches, $\alpha$ will be $1 / \sqrt n$. It makes the pre-activations with mean of 0 and standard deviation of $V_{th}$ at the early training. The codes of the tdBN can be found in \textbf{Supplementary Material A}.

  \subsection{Overall training algorithm}
  In this section, we present the overall training algorithm of the STBP-tdBN for training deep SNNs from scratch with our tdBN.

  In the error backpropagation, we consider the last layer as the decoding layer, and the final output $Q$ can be determined by
  \begin{align} Q &= \frac{1}{T}\sum_{t=1}^T{Mo^{n,t}},\end{align}
  where $o^{n,t}$ is the spike fired by the last output layer, $M$ is the matrix of decoding layer and $T$ denotes the number of timesteps.

  Then we make the outputs pass through a softmax layer. The loss function is determined as the cross-entropy. Considering the output $Q = (q_1,q_2,\cdots,q_n)$ and label vector $Y=(y_1,y_2,\cdots,y_n)$, loss function $L$ is defined by
  \begin{align} p_i &= \frac{e^{q_i}}{\sum_{j=1}^{n}e^{q_j}}, \\
     L &= -\sum_{i=1}^{n}y_ilog(p_i). \end{align}
  With the iterative LIF model, STBP-tdBN method backpropagates the gradient of the loss $L$ on both spatial and temporal domains. By applying the chain rule, $\frac{\partial L}{\partial o_i^{t,n}}$ and $ \frac{\partial L}{\partial u_i^{t,n}}$ can be computed by
  \begin{align}\frac{\partial L}{\partial o_i^{t,n}} &= \sum_{j=1}^{l(n+1)}\frac{\partial L}{\partial u_j^{t,n+1}}\frac{\partial u_j^{t,n+1}}{\partial o_i^{t,n}}+\frac{\partial L}{\partial u_i^{t+1,n}}\frac{\partial u_i^{t+1,n}}{\partial o_i^{t,n}},\\
  \frac{\partial L}{\partial u_i^{t,n}} &= \frac{\partial L}{\partial o_i^{t,n}}\frac{\partial o_i^{t,n}}{\partial u_i^{t,n}}+\frac{\partial L}{\partial u_i^{t+1,n}}\frac{\partial u_i^{t+1,n}}{\partial u_i^{t,n}},\end{align}
  where $o^{t,n}$ and $u^{t,n}$ represent the spike and membrane potential of the neuron in layer $n$ at time $t$.
  Because of the non-differentiable spiking activities, $ \frac{\partial o^t}{\partial u^t}$ doesn't exist in fact. To solve this problem, \cite{wu2018spatio} proposes derivative curve to approximate the derivative of spiking activity. In this paper, we use the rectangular function, which is proved to be efficient in gradient descent and can be determined by
  \begin{equation}  \frac{\partial o^t}{\partial u^t}= \frac{1}{a}sign(|u^t-V_{th}|<\frac{a}{2}).\end{equation}
  The codes of the overall training algorithm are shown in \textbf{Supplementary Material A}.
  \section{Theoretical Analysis}
  In this section, we will analyze the effects of tdBN to SNNs trained by STBP-tdBN. With theoretical tools about gradient norm theory in ANNs, we find that our tdBN can alleviate the problem of gradient vanishing or explosion during the training process. We also explain the functions of scaling factors $\alpha$ and $V_{th}$ we added during the normalization.
  \subsection{Gradient norm theory}
  The gradient norm theory has been developed well in recent years, which aims to conquer the gradient vanishing or explosion in various neural networks structure. In this paper, we adopt the ``\textbf{Block Dynamical Isometry}" proposed by \cite{chen2020comprehensive} to analyze tdBN's effect in directly-trained SNNs.  It considers a network as a series of blocks
  \begin{equation} f(x) = f_{i,\theta_i}\circ f_{i-1,\theta_i-1} \circ \cdots \circ f_{1,\theta_1}(x),\end{equation}
  where the function $f_{j,\theta_j}$ represents the $j^{th}$ block and define its input-output jacobian matrix as $\frac{\partial f_j}{\partial f_{j-1}}=J_j$. Let $\phi(J)$ represent the expectation of $tr(J)$ and $\varphi(J)$ denote $\phi(J^2)-\phi^2(J)$. Then they proved the following lemmas.
  \newtheorem{lemma1}{Lemma}
  \begin{lemma1}
     Consider a neural network that can be represented as a series of blocks as Eq. (17) and the $j^{th}$ block's jacobian matrix is denoted as $J_j$. If   $\ \ {\forall}j, \ \ \phi(J_j{J_j}^T) \approx 1$ and $\varphi(J_j{J_j}^T) \approx 0$ , the network achieves ``\textbf{Block Dynamical Isometry}" and can avoid gradient vanishing or explosion. \cite{chen2020comprehensive}
  \end{lemma1}
  \newtheorem{lemma2}[lemma1]{Lemma}
  \begin{lemma2}
  Consider a block of neural network, which consists of data normalization with 0-mean, linear transform and rectifier activations (`` \textbf{General Linear Transform}"). Let $2^{nd}$ moment of input vector as $\alpha_{in}$ and the output vector as $\alpha_{out}$, we have $ \phi(J{J}^T) = \frac{\alpha_{out}}{\alpha_{in}}$. \cite{chen2020comprehensive}
  \end{lemma2}

  Based on the theoretical framework of gradient norm, we combine it with the unique properties of spiking neurons and further analyze the effectiveness of our proposed tdBN algorithm for SNNs.

  LIF model has two special hyper-parameters: $\tau_{decay}$ and $V_{th}$, where $\tau_{decay}$ influences the gradients propagation in temporal domain and $V_{th}$ effects the spatial dimension. For experiment with SNNs, the $\tau_{decay}$ are often set as small value (e.g. 0.25). To analyze the gradient transformation, we simplify the model and set $\tau_{decay}$ as zero and we can get the following proposition.
  \newtheorem{theorem}{Theorem}
  \begin{theorem}
  Consider an SNN with $T$ timesteps and the $j^{th}$ block's jacobian matrix at time $t$ is denoted as $J_j^t$. When $\tau_{decay}$ is equal to 0, if we fix the second moment of input vector and the output vector to $V_{th}^2$ for each block between two tdBN layers, we have $\phi(J_j^t(J_j^t)^T) \approx 1$ and the training of SNN can avoid gradient vanishing or explosion.
  \end{theorem}
  \begin{proof}
  The proof of \textbf{Theorem 1} is based on the \textbf{Lemma 1} and \textbf{Lemma 2}. The details are presented in \textbf{Supplementary Material B}.
  \end{proof}
  \subsubsection{Influence of membrane decay mechanism}

  We analyze the effect of $\tau_{decay}$ to the gradient propagation. From equations (2) and (15), we have
  \begin{equation}\frac{\partial L}{\partial u_i^{t,n}} = \frac{\partial L}{\partial o_i^{t,n}}\frac{\partial o_i^{t,n}}{\partial u_i^{t,n}}+\frac{\partial L}{\partial u_i^{t+1,n}}\tau_{decay}(1-o_i^{t,n}).\end{equation}
  That is to say, if a neuron fires a spike, $(1-o_i^{t,n})$ is equal to zero and the gradient is irrelevant to $\tau_{decay}$. On the other hand, for $\tau_{decay}$ is a tiny constant, the gradient of the neuron at time $t+1$ has little influence to that at time $t$.

  To verify the \textbf{Theorem 1} and our analysis about influence of membrane decay mechanism, we evaluate our tdBN in 20-layers plain spiking network on CIFAR-10. In Fig. 2 , we display the mean of gradient norms in each layer during the first 1/6 epoch of training process despite the first encoding layer and the last decoding layer. And we find that when $\tau_{decay}=0$, the curve of gradient norm behaves much steady, which perfectly supports our theory. It should be emphasized that $\tau_{decay}$ cannot be set as 0 because it will prevent the information from propagating along temporal dimension and cause serious degradation. So, we evaluate our method in the condition that $\tau_{decay} \neq 0$. For example, when $\tau_{decay}=0.25$ and $0.5$, the gradient norm increases very slowly as the network deepens, which will not influence the training process. The results strongly support our results that can avoid gradient vanishing or explosion in deep SNNs.
\begin{figure}[!htp]
  \centering
  \includegraphics[width=0.5\textwidth]{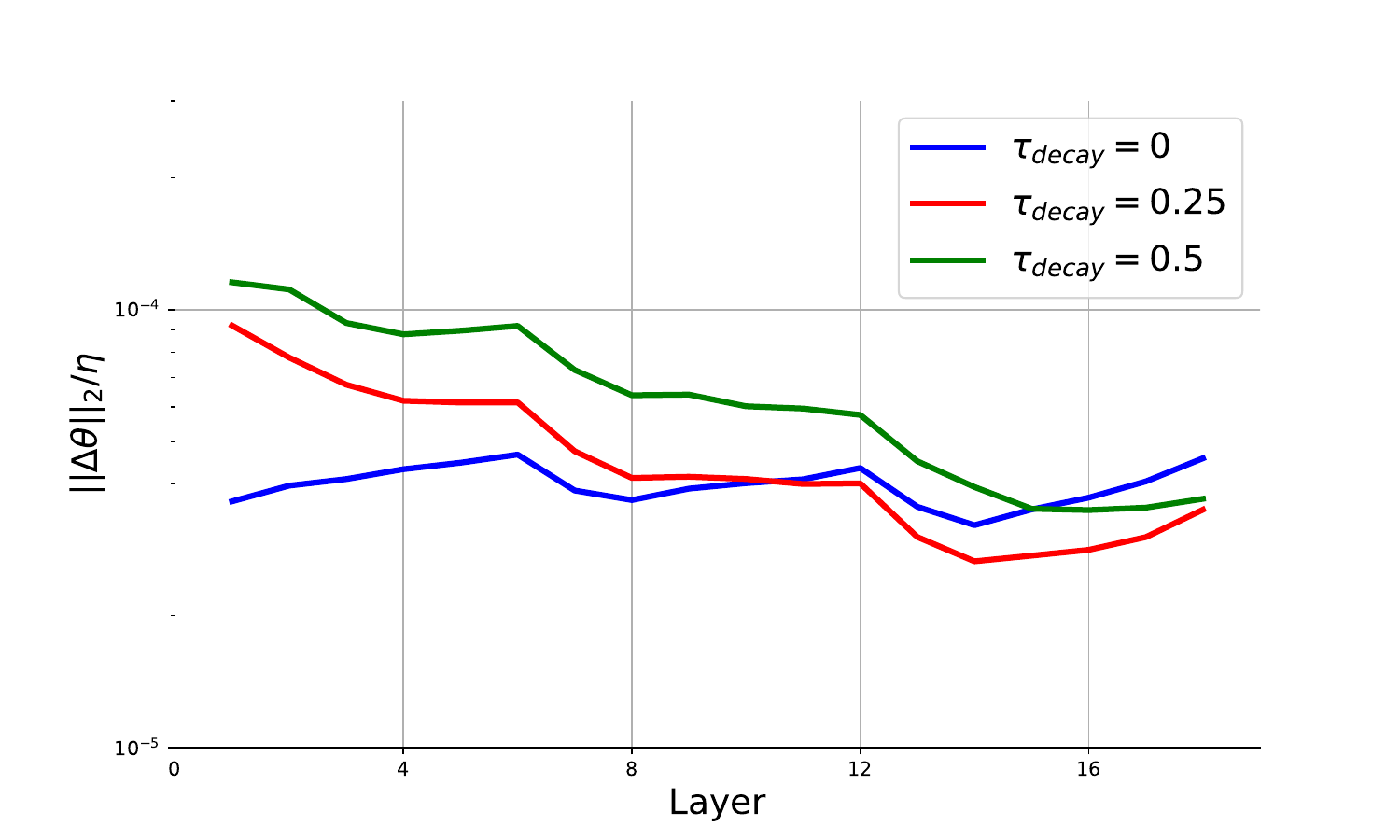}
  \caption{ \textbf{Gradient norm throughout the plain network with tdBN.}}
  \end{figure}

  \subsection{Scaling factors}  As we all know, a key for SNN model to obtain competitive performance is to set suitable threshold to maintain firing rates and reduce information loss. To achieve this, we introduce two scaling factors to the normalization implements in tdBN, which is meant to balance the pre-activations and threshold. In the early training,  with two scaling factors---$\alpha$ and $V_{th}$, we normalize the pre-activations to $N(0,{V_{th}^2})$ by initializing the trainable parameters $\lambda$ and $\beta$ with 1 and 0.

  First, we propose \textbf{Theorem 2} to explain the relations between per-activations and membrane potential of neurons, which helps to understand why our method works.
  \newtheorem{theorem1}[theorem]{Theorem}
  \begin{theorem1}
  With the iterative LIF model, assuming the pre-activations $x^t  \sim N(0,\sigma_{in}^2)$, we have the membrane potential $u^t \sim N(0,\sigma_{out}^2)$ and $\sigma_{out}^2 \propto \sigma_{in}^2$.
  \end{theorem1}
  \begin{proof}
  The proof of \textbf{Theorem 2} is presented in \textbf{Supplementary Material B}.
  \end{proof}
 \begin{figure}[!htp]
  \centering
  \includegraphics[width=0.5\textwidth]{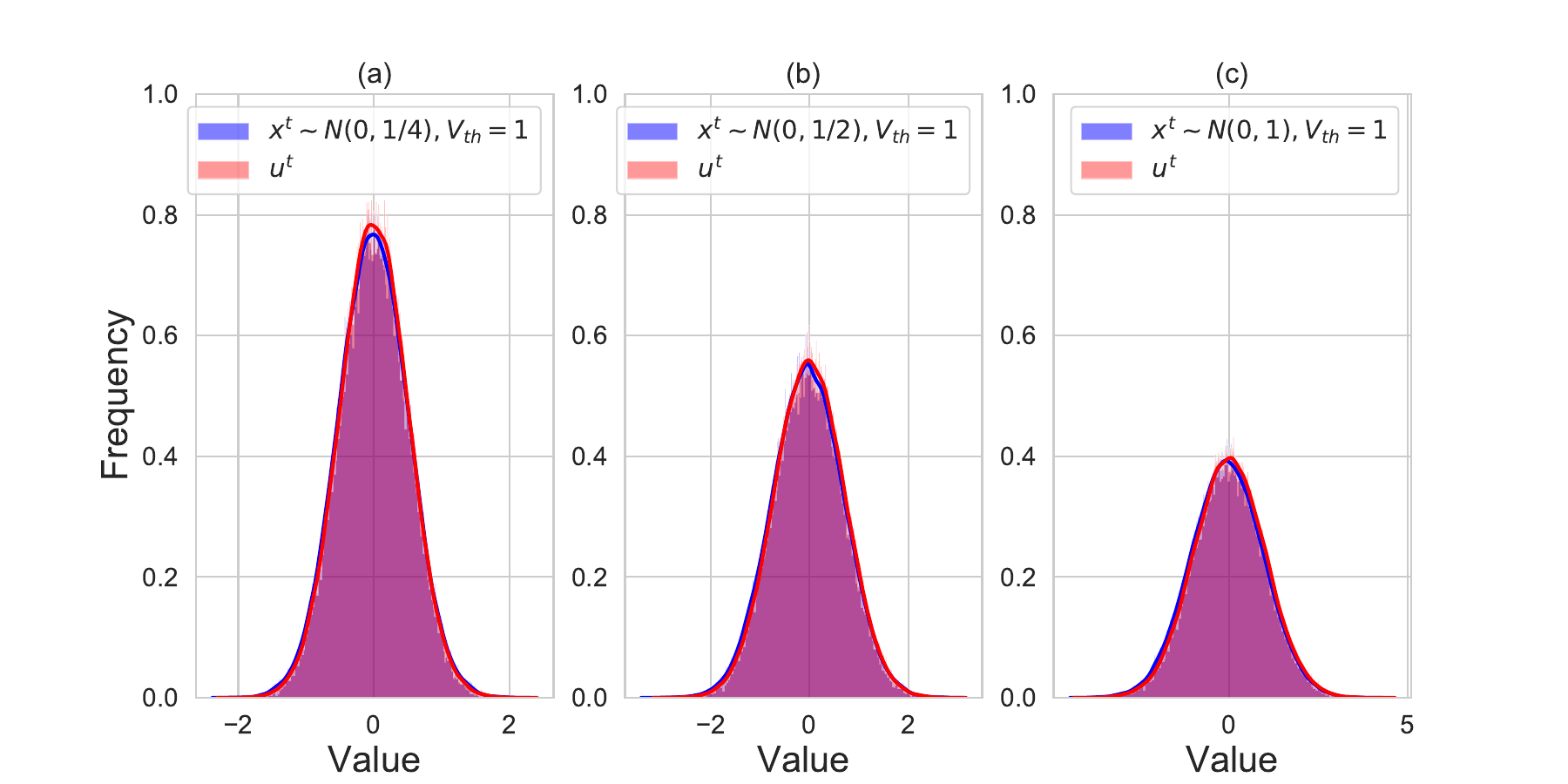}
  \caption{\textbf{Distributions of membrane potential $u^t$ with different variances of pre-activations $x^t$} }
  \end{figure}

  We verify the \textbf{Theorem 2} by visualized analysis. In the experiment, we set $\tau_{decay}=0.25$ and display the distribution of membrane potential with different pre-activations variance $\sigma_{in}^2$. The results are shown in Fig. 3. We find the high degree of similarity between the distributions of pre-activations and membrane potential, which supports the proposition.

  Next, we analyze the forward information propagation mechanism with the LIF model.
\begin{figure}[!htp]
  \centering
  \includegraphics[width=0.5\textwidth]{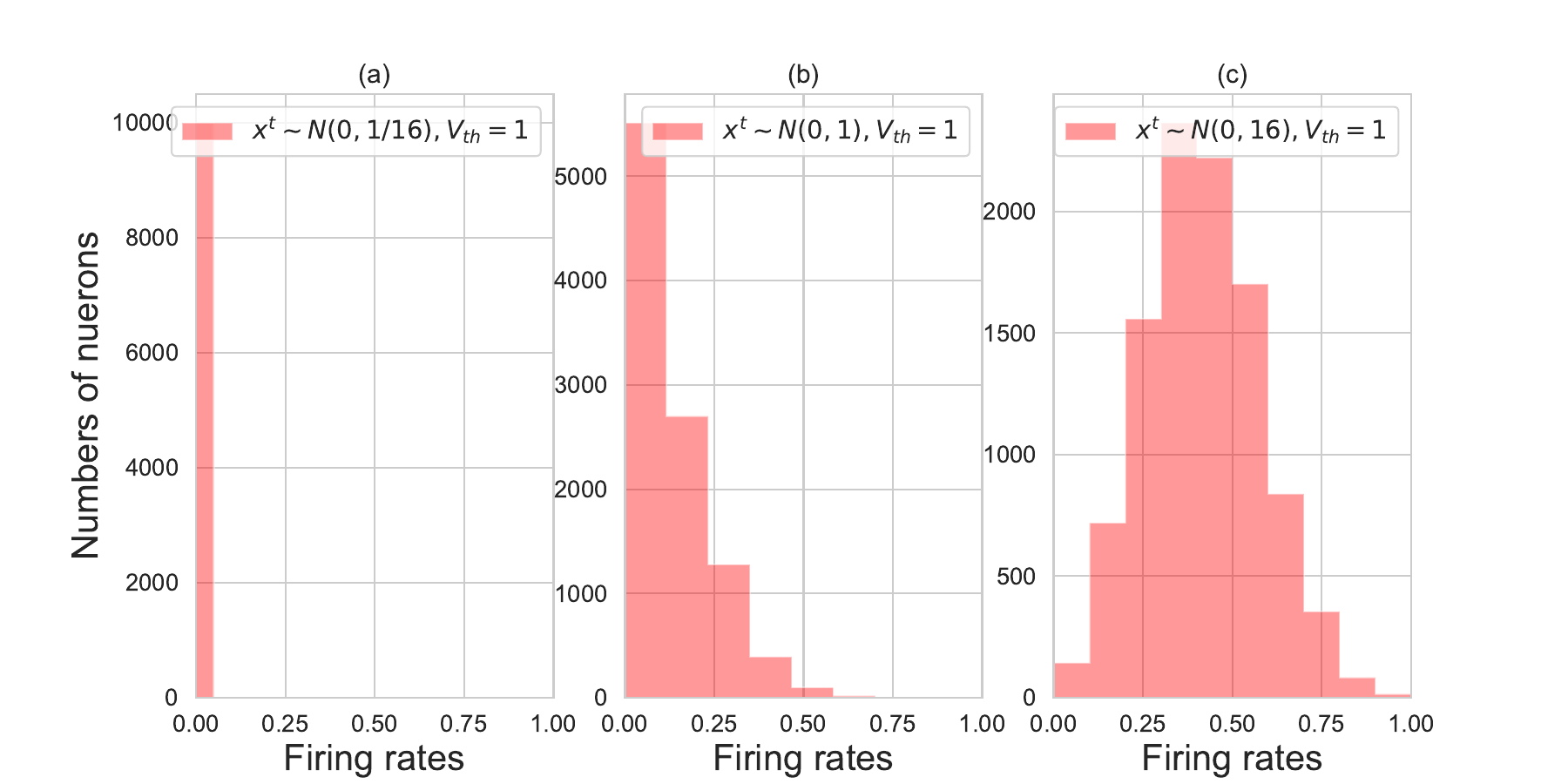}
  \caption{\textbf{Distributions of neurons' firing rates with different variances of pre-activations $x^t$.} }
  \end{figure}

  During the forward, when membrane potential reaches the threshold, neurons will fire a spike and make the information propagate layer by layer. With \textbf{Theorem 2} and Eq. (3) , we can approximate the possibilities of neurons firing a spike $P( u^t>V_{th} )$. It's obvious that it is a positive correlation between $P( u^t>V_{th} )$ and the variance of membrane potential $\sigma_{out}^2$ as well as $\sigma_{in}^2$. So, we adopt the scaling factors to adjust the distributions of pre-activations to maintain the firing rates in deep SNNs. Fig. 4 shows the distribution of neurons' firing rates when we set variances of pre-activations $x^t \sim N(0,\sigma_{in}^2)$ as different values. Because of the decay mechanisms, even a neuron receives positive input each time, it may fire no spike (Fig. 4(a)), which means neurons in next layer only receive few nonzero pre-synaptic inputs, making spikes disappear in deep SNNs and preventing signal from propagating forward. Another situation is that a neuron fires spikes all the time (Fig. 4(c)), which means the outputs of some neurons to be insensitive to the change of pre-activations and causes increase of computation.

  In conclusion, to balance pre-synaptic input and threshold to maintain the firing rates, we utilize the scaling factors to control the variance of membrane potential and pre-activations, which alleviates its dependence on the threshold. So, we normalize the pre-activations to $N(0,{V_{th}^2})$.

\section{Deep Spiking Residual Network}
ResNet is one of the most popular architectures o tackle with the problem of degradation when networks go deep. With the shortcut connections, \cite{he2016deep} adds identity mapping between layers, which enables the training of very deep neural networks. Inspired by the residual learning, we propose our deep spiking residual network, which replaces the BN layer with our tdBN and changes the shortcut connection to achieve better performance.

\subsubsection{Basic block}
ResNet in ANNs is built with some basic blocks. Fig. 5(a) shows a form of traditional basic blocks in ResNet-ANN. It is implemented in relatively shallow residual networks, which consists of two layers with $3\times3$ convolution kernel followed by a BN layer and the ReLU activation. On the basis, we propose the basic block of our deep spiking residual network. As shown in Fig. 5(b), we replace the ReLU activation with LIF model and replace the BN layer with our tdBN. Moreover, we change the shortcut connection and add a tdBN layer before the final addition. Hence, expect that the hyper-parameters $\alpha$ in tdBN layers before the final activation layer or in the shortcut are set as $1/\sqrt2$, the other tdBN layers' hyper-parameter $\alpha$ is defined as 1. It guarantees the input distribution of each activation will satisfy $N(0,V_{th}^2)$ at the beginning of training.
\begin{figure}[!htp]
\centering
\includegraphics[width=0.5\textwidth]{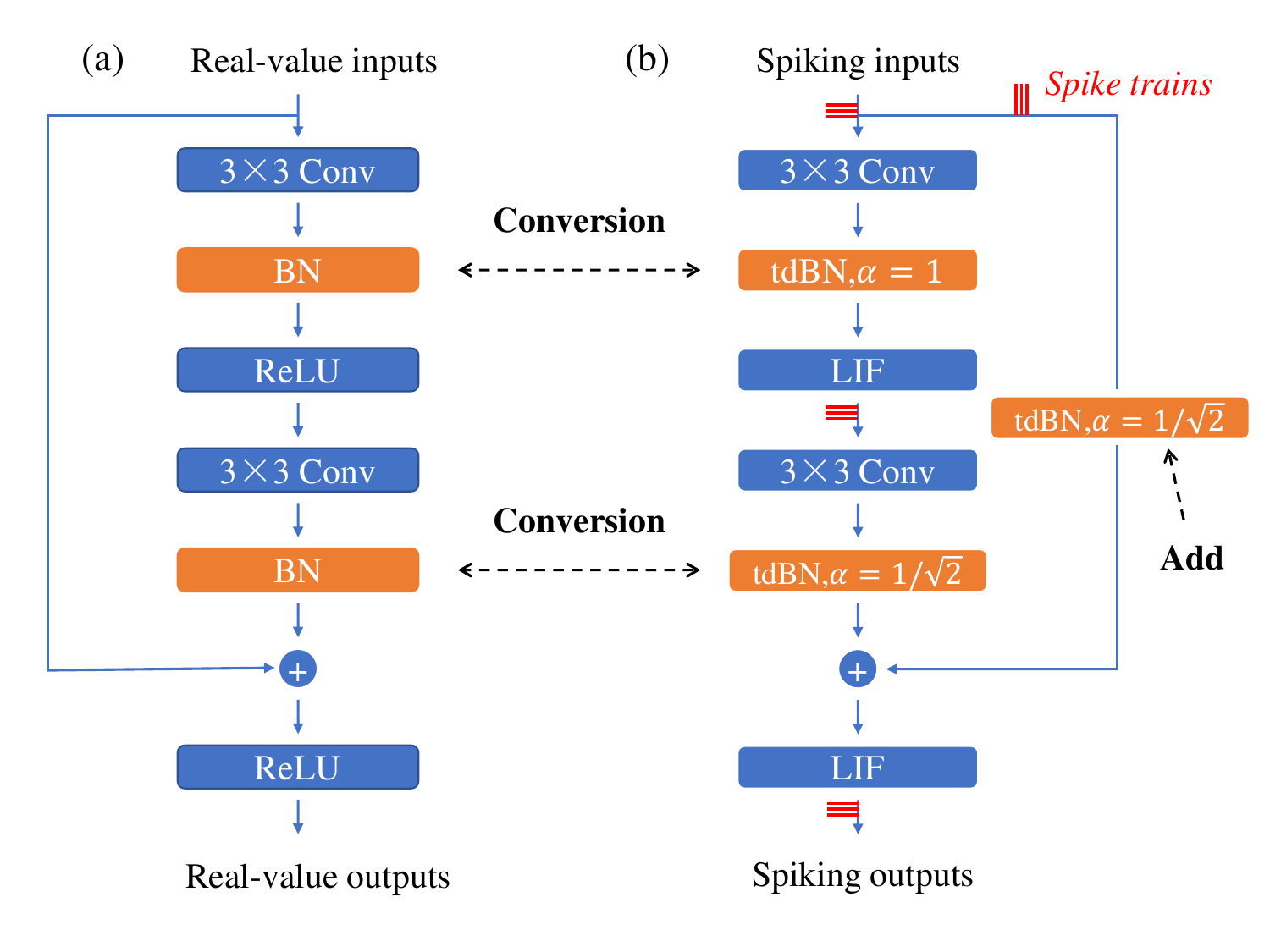}
\caption{\textbf{Different basic blocks between traditional ResNet-ANN (a) and our ResNet-SNN with tdBN (b).}}
\end{figure}

\subsubsection{Network architecture}
We build the deep spiking Residual network with our basic blocks. The first encoding layer receives the inputs and performs downsampling directly with convolution kernel and a stride of 2. Then, the spiking activities propagate through basic blocks. Similar to ResNet-ANN, we double the number of channels when the feature map is halved. After average pooling or full-connected layer if necessary, the last decoding layer is a fully-connected layer followed by softmax function.

\begin{table*}[!htbp]
\centering

\resizebox{\textwidth}{!}{%
\begin{tabular}{clclclclclc}
\hline
\hline
Dataset                                        &  & Model                               &  & Methods                    &  & Architecture               &  & TimeStep                    &  & Accuracy         \\ \hline
\multicolumn{1}{c|}{\multirow{9}{*}{CIFAR-10}} &  &\cite{sengupta2019going}             &  & ANN-SNN                    &  & VGG-16                     &  & 2500                        &  & 91.55\%          \\
\multicolumn{1}{c|}{}                          &  & \cite{hu2018spiking}                    &  & ANN-SNN                    &  & ResNet-44                  &  & 350                         &  & 92.37\%          \\
\multicolumn{1}{c|}{}                          &  & \cite{rathi2020enabling}                 &  & Hybird Training            &  & VGG-16                     &  & 200                         &  & 92.02\%          \\
\multicolumn{1}{c|}{}                          &  & \cite{lee2020enabling}                  &  & Spike-based BP             &  & ResNet-11                  &  & 100                         &  & 90.95\%          \\
\multicolumn{1}{c|}{}                          &  & \cite{wu2019direct}                    &  & STBP                       &  & 5 Conv, 2Fc                &  & 12                          &  & 90.53\%          \\ \cline{2-11}
\multicolumn{1}{c|}{}                          &  & \multirow{3}{*}{\textbf{our model}} &  & \multirow{3}{*}{STBP-tdBN} &  & \multirow{3}{*}{ResNet-19} &  & \textbf{6}                  &  & \textbf{93.16\%} \\
\multicolumn{1}{c|}{}                          &  &                                     &  &                            &  &                            &  & \textbf{4}                  &  & \textbf{92.92\%} \\
\multicolumn{1}{c|}{}                          &  &                                     &  &                            &  &                            &  & \textbf{2}                  &  & \textbf{92.34\%} \\ \hline
\multicolumn{1}{c|}{\multirow{9}{*}{ImageNet}} &  & \cite{sengupta2019going}            &  & ANN-SNN                    &  & VGG-16                     &  & 2500                        &  & 69.96\%          \\
\multicolumn{1}{c|}{}                          &  & \cite{sengupta2019going}             &  & ANN-SNN                    &  & ResNet-34                  &  & 2500                        &  & 65.47\%          \\
\multicolumn{1}{c|}{}                          &  & \cite{han2020rmp}     &  & ANN-SNN                    &  & ResNet-34                  &  & 1024                        &  & 66.61\%          \\
\multicolumn{1}{c|}{}                          &  & \cite{hu2018spiking}                   &  & ANN-SNN                    &  & ResNet-34                  &  & 768                         &  & 71.61\%          \\
\multicolumn{1}{c|}{}                          &  & \cite{rathi2020enabling}                &  & Hybird Training            &  & VGG-16                     &  & 250                         &  & 65.19\%          \\
\multicolumn{1}{c|}{}                          &  & \cite{rathi2020enabling}                &  & Hybird Training            &  & ResNet-34                  &  & 250                         &  & 61.48\%          \\ \cline{2-11}
\multicolumn{1}{c|}{}                          &  & \multirow{3}{*}{\textbf{our model}} &  & \multirow{3}{*}{STBP-tdBN} &  & ResNet-34                  &  & \textbf{6} &  & \textbf{63.72\%} \\
\multicolumn{1}{c|}{}                          &  &                                     &  &                            &  & ResNet-50                  &  &  \textbf{6}                           &  & \textbf{64.88\%} \\
\multicolumn{1}{c|}{}                          &  &                                     &  &                            &  & ResNet-34(large)           &  &   \textbf{6}                          &  & \textbf{67.05\%}\\
\hline
\hline

\end{tabular}%
}\caption{Comparisons with SNNs on CIFAR-10 \& ImageNet}
\end{table*}

\section{Experiment}
We test our deep residual SNNs on static non-spiking datasets (CIFAR-10, ImageNet) and neuromorphic datasets (DVS-gesture, DVS-CIFAR10). And we compare our results with SNN models to demonstrate the advantages of our method on both accuracy and the number of timesteps. The dataset introduction, pre-processing, parameter configuration, training details and results analysis are summarized in \textbf{Supplementary Material C}.

\subsection{Experiment on static  datasets}
Experiments on static  datasets include CIFAR-10 and ImageNet, which serve as standard image recognition benchmarks. We test our ResNet-SNN with different timesteps, sizes and depths. With very few timesteps, our models significantly reduce the amount of computation compared with ANNs with same architectures due to sparse spikes. The detailed analysis is presented in \textbf{Supplementary Material D}.

\subsubsection{CIFAR-10}

CIFAR-10 is an image dataset with 50000 training images and 10000 testing images with size of $32 \times 32$, which all belong to 10 classes. In this experiment, we use ResNet-19 with timesteps of 2, 4, 6.

Our results are shown in the Table 1. Before us, SNNs using the ANN-SNN converted method report best accuracy on CIFAR-10. However, to achieve good performance, converted SNNs usually require more than 100 timesteps.  In this work, our model achieves state-of-the-art performance (\textbf{93.15\% top-1 accuracy with only 6 timesteps}) on CIFAR-10, which not only greatly reduces latency and computation cost compared with other SNN models.

% \usepackage{multirow}
% \usepackage{graphicx}

\subsubsection{ImageNet}
ImageNet \cite{deng2009imagenet} contains 1.28 million training images and 50000 validating images. On ImageNet, we test ResNet-34 with standard size and large size. The large model doubles the channels compared with the standard one and achieves \textbf{67.05\% top-1 accuracy with just 6 timesteps}. Also, we use ResNet-50 to explore the very deep directly-trained SNNs and achieve 64.88\% top-1 accuracy. Similar to ResNet-ANN, we observe the enhancement of accuracy between ResNet-34 and ResNet-50, which shows the deeper residual spiking networks may perform better in complex recognition tasks.

Before us, the directly-trained full-spiking SNNs have never reported competitive results on ImageNet while our methods make a breakthrough. Compared with early works, our methods achieve higher accuracies with fewer timesteps. The results are shown in Table 1.

\subsection{Experiment on neuromorphic datasets}
Compared with non-spiking static datasets, neuromorphic datasets contain more temporal information, which are more suitable for SNNs to demonstrate their advantages. Previous work \cite{wu2019direct} only tests the MLP-SNN with STBP on N-MNIST and report the state-of-art results (98.57\% accuracy). However, N-MNIST is too simple for our deep directly-trained SNNs with residual structures. So, we adopt two challenging neuromorphic datasets---DVS-Gesture and DVS-CIFAR10.

\subsubsection{DVS-Gesture}DVS-Gesture \cite{amir2017low} is a collection of moving gestures performed by 29 different individuals belonging to 11 classes, which is captured by DVS cameras under three lighting conditions.  We set timestep $T$ to be 40. In each timestep, the network receives only one slice of the event stream, which means the first $1200 ms$ of each action is used for training or testing. In this experiment, we use ResNet-17 and achieve an accuracy of 96.87\%, which is the state-of-the-art result for directly-trained SNNs on DVS-Gesture. We compare our results with other related works on DVS-Gesture, as shown in Table 2.
\begin{table}[H]
	\centering

\resizebox{0.45\textwidth}{!}{%
	\begin{tabular}{cccc}
        \hline
		\hline
		Model&Methods&Accuracy\\
		\hline
        \cite{he2020comparing}&STBP&93.40\%\\
        \cite{shrestha2018slayer}&SLAYER&93.64\%\\
        \cite{kugele2020efficient}&ANN&95.68\%\\
        \cite{amir2017low}&BPTT&94.59\%\\
        \hline
        \textbf{our model}&STBP-tdBN&\textbf{96.87\%}\\
        \hline
        \hline
	\end{tabular}}
\caption{Accuracy comparisons on DVS-Gesture}
\end{table}

\subsubsection{DVS-CIFAR10}
  DVS-CIFAR10 \cite{li2017cifar10} is a neuromorphic dataset converted from famous CIFAR-10 to its dynamic form. It consists of 1000 images in the format of spike train per class. With the noisy environment, it's also a challenging recognition task similar to DVS-Gesture. With ResNet-19, our methods achieve the best performance with 67.8\% accuracy in 10 timesteps. Table 3 compares our results with other models.
  \begin{table}[H]
	\centering

\resizebox{0.45\textwidth}{!}{%
	\begin{tabular}{cccc}
        \hline
		\hline
		Model&Methods&Accuracy\\
		\hline
        \cite{sironi2018hats}&HATS&52.40\%\\
        \cite{kugele2020efficient}&Streaming rollout ANN&66.75\%\\
        \cite{ramesh2019dart}&DART&65.78\%\\
        \cite{wu2019direct}&STBP&60.5\%\\
        \hline
        \textbf{our model}&STBP-tdBN&\textbf{67.8\%}\\
        \hline
        \hline
	\end{tabular}}
    \caption{Accuracy comparisons on DVS-CIFAR10}
\end{table}

\section{Conclusion}
In this paper, we present a normalization method enabling directly-trained deep SNNs with high performance. We combine the gradient norm theory and prove that this method can effectively balance the input stimulus and neuronal threshold during training, thereby facilitating the learning convergence.  On this basis, by further introducing the shortcut connection, we greatly extend directly-trained SNNs from a common shallow structure (less than ten layers) to a very deep structure (more than fifty layers). Finally, the model is evaluated on both large-scale static image datasets and neuromorphic datasets. Comparing with other SNN models, we achieve a high accuracy on CIFAR-10 and ImageNet with a significantly small inference latency. To our best knowledge, it is the first work to report a directly-trained and very deep SNNs on ImageNet. On neuromorphic datasets, our model can efficiently process temporal-spatial information and achieve state-of-the-art performance.

In summary, this work provides a feasible directly-trained scheme for deep SNNs. It maintains high efficiency of spike-based communication mechanism and enables SNNs to solve more complex large-scale classification tasks, which may benefit the implementations on the neuromorphic hardware and promote the practical applications of SNNs.
\section{Acknowledgment}
This work is partially supported by National Key R\&D Program of China (No.2018YFE0200200,2018AAA0102600), Beijing Academy of Artificial Intelligence (BAAI), and a grant from the Institute for Guo Qiang of Tsinghua university, and in part by the Science and Technology Major Project of Guangzhou (202007030006), and the key scientific technological innovation research project by Ministry of Education, and the open project of Zhejiang laboratory.
\bibliography{ref}

\begin{thebibliography}{28}
\providecommand{\natexlab}[1]{#1}
\providecommand{\url}[1]{\texttt{#1}}
\providecommand{\urlprefix}{URL }
\expandafter\ifx\csname urlstyle\endcsname\relax
  \providecommand{\doi}[1]{doi:\discretionary{}{}{}#1}\else
  \providecommand{\doi}{doi:\discretionary{}{}{}\begingroup
  \urlstyle{rm}\Url}\fi

\bibitem[{Amir et~al.(2017)Amir, Taba, Berg, Melano, McKinstry, Di~Nolfo,
  Nayak, Andreopoulos, Garreau, Mendoza, Kusnitz, Debole, Esser, Delbruck,
  Flickner, and Modha}]{amir2017low}
Amir, A.; Taba, B.; Berg, D.; Melano, T.; McKinstry, J.; Di~Nolfo, C.; Nayak,
  T.; Andreopoulos, A.; Garreau, G.; Mendoza, M.; Kusnitz, J.; Debole, M.;
  Esser, S.; Delbruck, T.; Flickner, M.; and Modha, D. 2017.
\newblock A Low Power, Fully Event-Based Gesture Recognition System.
\newblock In \emph{Proceedings of the IEEE Conference on Computer Vision and
  Pattern Recognition (CVPR)}.

\bibitem[{Ba, Kiros, and Hinton(2016)}]{layernorm}
Ba, J.~L.; Kiros, J.~R.; and Hinton, G.~E. 2016.
\newblock Layer Normalization.
\newblock \emph{arXiv preprint arXiv:1607.06450} .

\bibitem[{Chen et~al.(2020)Chen, Deng, Wang, Li, and
  Xie}]{chen2020comprehensive}
Chen, Z.; Deng, L.; Wang, B.; Li, G.; and Xie, Y. 2020.
\newblock A Comprehensive and Modularized Statistical Framework for Gradient
  Norm Equality in Deep Neural Networks.
\newblock \emph{IEEE Transactions on Pattern Analysis and Machine Intelligence}
  1--1.

\bibitem[{{Davies} et~al.(2018){Davies}, {Srinivasa}, {Lin}, {Chinya}, {Cao},
  {Choday}, {Dimou}, {Joshi}, {Imam}, {Jain}, {Liao}, {Lin}, {Lines}, {Liu},
  {Mathaikutty}, {McCoy}, {Paul}, {Tse}, {Venkataramanan}, {Weng}, {Wild},
  {Yang}, and {Wang}}]{davies2018loihi}
{Davies}, M.; {Srinivasa}, N.; {Lin}, T.; {Chinya}, G.; {Cao}, Y.; {Choday},
  S.~H.; {Dimou}, G.; {Joshi}, P.; {Imam}, N.; {Jain}, S.; {Liao}, Y.; {Lin},
  C.; {Lines}, A.; {Liu}, R.; {Mathaikutty}, D.; {McCoy}, S.; {Paul}, A.;
  {Tse}, J.; {Venkataramanan}, G.; {Weng}, Y.; {Wild}, A.; {Yang}, Y.; and
  {Wang}, H. 2018.
\newblock Loihi: A Neuromorphic Manycore Processor with On-Chip Learning.
\newblock \emph{IEEE Micro} 38(1): 82--99.

\bibitem[{{Deng} et~al.(2009){Deng}, {Dong}, {Socher}, {Li}, {Kai Li}, and {Li
  Fei-Fei}}]{deng2009imagenet}
{Deng}, J.; {Dong}, W.; {Socher}, R.; {Li}, L.; {Kai Li}; and {Li Fei-Fei}.
  2009.
\newblock ImageNet: A large-scale hierarchical image database.
\newblock In \emph{2009 IEEE Conference on Computer Vision and Pattern
  Recognition}, 248--255.

\bibitem[{{Diehl} et~al.(2015){Diehl}, {Neil}, {Binas}, {Cook}, {Liu}, and
  {Pfeiffer}}]{diehl2015fast}
{Diehl}, P.~U.; {Neil}, D.; {Binas}, J.; {Cook}, M.; {Liu}, S.; and {Pfeiffer},
  M. 2015.
\newblock Fast-classifying, high-accuracy spiking deep networks through weight
  and threshold balancing.
\newblock In \emph{2015 International Joint Conference on Neural Networks
  (IJCNN)}, 1--8.

\bibitem[{Han, Srinivasan, and Roy(2020)}]{han2020rmp}
Han, B.; Srinivasan, G.; and Roy, K. 2020.
\newblock RMP-SNN: Residual Membrane Potential Neuron for Enabling Deeper
  High-Accuracy and Low-Latency Spiking Neural Network.
\newblock In \emph{Proceedings of the IEEE Conference on Computer Vision and
  Pattern Recognition (CVPR)}.

\bibitem[{He et~al.(2016)He, Zhang, Ren, and Sun}]{he2016deep}
He, K.; Zhang, X.; Ren, S.; and Sun, J. 2016.
\newblock Deep Residual Learning for Image Recognition.
\newblock In \emph{Proceedings of the IEEE Conference on Computer Vision and
  Pattern Recognition (CVPR)}.

\bibitem[{He et~al.(2020)He, Wu, Deng, Li, Wang, Tian, Ding, Wang, and
  Xie}]{he2020comparing}
He, W.; Wu, Y.; Deng, L.; Li, G.; Wang, H.; Tian, Y.; Ding, W.; Wang, W.; and
  Xie, Y. 2020.
\newblock Comparing SNNs and RNNs on Neuromorphic Vision Datasets: Similarities
  and Differences.
\newblock \emph{arXiv preprint arXiv:2005.02183} .

\bibitem[{Hu et~al.(2018)Hu, Tang, Wang, and Pan}]{hu2018spiking}
Hu, Y.; Tang, H.; Wang, Y.; and Pan, G. 2018.
\newblock Spiking Deep Residual Network.
\newblock \emph{arXiv preprint arXiv:1805.01352} .

\bibitem[{Hunsberger and Eliasmith(2015)}]{hunsberger2015spiking}
Hunsberger, E.; and Eliasmith, C. 2015.
\newblock Spiking Deep Networks with LIF Neurons.
\newblock \emph{arXiv preprint arXiv:1510.08829} .

\bibitem[{Ioffe and Szegedy(2015)}]{ioffe2015batchnorm}
Ioffe, S.; and Szegedy, C. 2015.
\newblock Batch Normalization: Accelerating Deep Network Training by Reducing
  Internal Covariate Shift.
\newblock \emph{arXiv preprint arXiv:1502.03167} .

\bibitem[{Jin, Zhang, and Li(2018)}]{jin2018hybrid}
Jin, Y.; Zhang, W.; and Li, P. 2018.
\newblock Hybrid Macro/Micro Level Backpropagation for Training Deep Spiking
  Neural Networks.
\newblock In \emph{Advances in Neural Information Processing Systems 31},
  7005--7015.

\bibitem[{Kugele et~al.(2020)Kugele, Pfeil, Pfeiffer, and
  Chicca}]{kugele2020efficient}
Kugele, A.; Pfeil, T.; Pfeiffer, M.; and Chicca, E. 2020.
\newblock Efficient Processing of Spatio-Temporal Data Streams With Spiking
  Neural Networks.
\newblock \emph{Frontiers in Neuroscience} 14: 439.

\bibitem[{Lee et~al.(2020)Lee, Sarwar, Panda, Srinivasan, and
  Roy}]{lee2020enabling}
Lee, C.; Sarwar, S.~S.; Panda, P.; Srinivasan, G.; and Roy, K. 2020.
\newblock Enabling Spike-Based Backpropagation for Training Deep Neural Network
  Architectures.
\newblock \emph{Frontiers in Neuroscience} 14: 119.

\bibitem[{Lee, Delbruck, and Pfeiffer(2016)}]{lee2016training}
Lee, J.~H.; Delbruck, T.; and Pfeiffer, M. 2016.
\newblock Training Deep Spiking Neural Networks Using Backpropagation.
\newblock \emph{Frontiers in Neuroscience} 10: 508.

\bibitem[{Li et~al.(2017)Li, Liu, Ji, Li, and Shi}]{li2017cifar10}
Li, H.; Liu, H.; Ji, X.; Li, G.; and Shi, L. 2017.
\newblock CIFAR10-DVS: An Event-Stream Dataset for Object Classification.
\newblock \emph{Frontiers in Neuroscience} 11: 309.

\bibitem[{Merolla et~al.(2014)Merolla, Arthur, Alvarez-Icaza, Cassidy, Sawada,
  Akopyan, Jackson, Imam, Guo, Nakamura, Brezzo, Vo, Esser, Appuswamy, Taba,
  Amir, Flickner, Risk, Manohar, and Modha}]{merolla2014million}
Merolla, P.~A.; Arthur, J.~V.; Alvarez-Icaza, R.; Cassidy, A.~S.; Sawada, J.;
  Akopyan, F.; Jackson, B.~L.; Imam, N.; Guo, C.; Nakamura, Y.; Brezzo, B.; Vo,
  I.; Esser, S.~K.; Appuswamy, R.; Taba, B.; Amir, A.; Flickner, M.~D.; Risk,
  W.~P.; Manohar, R.; and Modha, D.~S. 2014.
\newblock A million spiking-neuron integrated circuit with a scalable
  communication network and interface.
\newblock \emph{Science} 345(6197): 668--673.

\bibitem[{{Ramesh} et~al.(2019){Ramesh}, {Yang}, {Orchard}, {Le Thi}, {Zhang},
  and {Xiang}}]{ramesh2019dart}
{Ramesh}, B.; {Yang}, H.; {Orchard}, G.~M.; {Le Thi}, N.~A.; {Zhang}, S.; and
  {Xiang}, C. 2019.
\newblock DART: Distribution Aware Retinal Transform for Event-based Cameras.
\newblock \emph{IEEE Transactions on Pattern Analysis and Machine Intelligence}
  1--1.

\bibitem[{Rathi et~al.(2020)Rathi, Srinivasan, Panda, and
  Roy}]{rathi2020enabling}
Rathi, N.; Srinivasan, G.; Panda, P.; and Roy, K. 2020.
\newblock Enabling Deep Spiking Neural Networks with Hybrid Conversion and
  Spike Timing Dependent Backpropagation.
\newblock \emph{arXiv preprint arXiv:2005.01807} .

\bibitem[{Roy, Jaiswal, and Panda(2019)}]{PMID:31776490}
Roy, K.; Jaiswal, A.; and Panda, P. 2019.
\newblock Towards spike-based machine intelligence with neuromorphic computing.
\newblock \emph{Nature} 575(7784): 607-617.

\bibitem[{Sengupta et~al.(2019)Sengupta, Ye, Wang, Liu, and
  Roy}]{sengupta2019going}
Sengupta, A.; Ye, Y.; Wang, R.; Liu, C.; and Roy, K. 2019.
\newblock Going Deeper in Spiking Neural Networks: VGG and Residual
  Architectures.
\newblock \emph{Frontiers in Neuroscience} 13: 95.

\bibitem[{Shrestha and Orchard(2018)}]{shrestha2018slayer}
Shrestha, S.~B.; and Orchard, G. 2018.
\newblock SLAYER: Spike Layer Error Reassignment in Time.
\newblock In Bengio, S.; Wallach, H.; Larochelle, H.; Grauman, K.;
  Cesa-Bianchi, N.; and Garnett, R., eds., \emph{Advances in Neural Information
  Processing Systems 31}, 1412--1421.

\bibitem[{Sironi et~al.(2018)Sironi, Brambilla, Bourdis, Lagorce, and
  Benosman}]{sironi2018hats}
Sironi, A.; Brambilla, M.; Bourdis, N.; Lagorce, X.; and Benosman, R. 2018.
\newblock HATS: Histograms of Averaged Time Surfaces for Robust Event-Based
  Object Classification.
\newblock In \emph{Proceedings of the IEEE Conference on Computer Vision and
  Pattern Recognition (CVPR)}, 1731--1740.

\bibitem[{Stromatias et~al.(2015)Stromatias, Neil, Pfeiffer, Galluppi, Furber,
  and Liu}]{stromatias2015robustness}
Stromatias, E.; Neil, D.; Pfeiffer, M.; Galluppi, F.; Furber, S.~B.; and Liu,
  S.-C. 2015.
\newblock Robustness of spiking Deep Belief Networks to noise and reduced bit
  precision of neuro-inspired hardware platforms.
\newblock \emph{Frontiers in Neuroscience} 9: 222.

\bibitem[{Wu et~al.(2018)Wu, Deng, Li, Zhu, and Shi}]{wu2018spatio}
Wu, Y.; Deng, L.; Li, G.; Zhu, J.; and Shi, L. 2018.
\newblock Spatio-Temporal Backpropagation for Training High-Performance Spiking
  Neural Networks.
\newblock \emph{Frontiers in Neuroscience} 12: 331.

\bibitem[{Wu et~al.(2019)Wu, Deng, Li, Zhu, Xie, and Shi}]{wu2019direct}
Wu, Y.; Deng, L.; Li, G.; Zhu, J.; Xie, Y.; and Shi, L. 2019.
\newblock Direct training for spiking neural networks: Faster, larger, better.
\newblock In \emph{Proceedings of the AAAI Conference on Artificial
  Intelligence}, volume~33, 1311--1318.

\bibitem[{Wu and He(2018)}]{wu2018groupnorm}
Wu, Y.; and He, K. 2018.
\newblock Group normalization.
\newblock In \emph{Proceedings of the European Conference on Computer Vision
  (ECCV)}, 3--19.

\end{thebibliography}
\newpage
\section{Supplementary Material}
\section{A\ \ \ \ Codes for Algorithms}
In this section, we present the detailed codes for (1) threshold-dependent batch normalization (\textbf{algorithm 1}) (2) overall training algorithm (\textbf{algorithm 2}).
\begin{algorithm}[h]
  \caption{Threshold-dependent BatchNorm}
  \textbf{In the Training:}\\
  \textbf{Input:} weighted outputs from last layer \\$x \in [Timestep,Batch\_size,C_{out},H,W].$\\
  \textbf{Output:} normalized  pre-synaptic inputs to the next layer\\
  $y \in [Timestep,Batch\_size,C_{out},H,W].$ \\
  \textbf{Parameters:} two trainable parameters $\lambda,\beta \in [C_{out}]$, threshold $V_{th}$, hyper-parameter $\alpha$.\\
  \textbf{Initialize:}$\lambda \leftarrow 1$, $\beta \leftarrow 0$.
  \begin{algorithmic}[1]
    \State $\mu_{tra} = mean(x[c_{out}:])$

    \State $\sigma^2_{tra} = mean(square(x[c_{out}:]-\mu_{tra}))$

    \State $\hat{x} = \frac{\alpha V_{th}(x-\mu_{tra})}{\sqrt{\sigma^2_{tra}+\epsilon}}$

    \State $y = \lambda\hat{x}+\beta$
  \end{algorithmic}

  \textbf{In the Inference:}\\
  \textbf{Input:} convolution kernel $W \in [C_{out},C_{in},H,W]$, bias $B \in [C_{out}]$, spiking outputs from last layer $o \in [Timestep,Batch\_size,C_{in},H,W].$\\
  \textbf{Output:} pre-synaptic inputs to the next layer  $y' \in [Timestep,Batch\_size,C_{out},H,W].$ \\
  \textbf{Parameters:} two trainable parameters from tdBN for training $\lambda,\beta \in [C_{out}]$, threshold $V_{th}$, hyper-parameter $\alpha$.
  \begin{algorithmic}[1]
    \State $\mu_{inf}\leftarrow E(\mu_{tra})$

    \State $\sigma^2_{inf}\leftarrow E(\sigma^2_{tra})$

    \State $W'=\lambda\frac{\alpha V_{th}W}{\sqrt{\sigma^2_{inf}+\epsilon}}$

    \State $B'=\lambda\frac{\alpha V_{th}(B-\mu_{inf})}{\sqrt{\sigma^2_{inf}+\epsilon}}+\beta$
    \For {$t = 1$ \textbf{to} $T$}
    \State ${y'[t]}={W}' \circledast o[t]+{B}'$
    \EndFor
  \end{algorithmic}
  \end{algorithm}
  \\
  \\

  \begin{algorithm}[]

  \caption{Overall Training Algorithm}

  \textbf{Input:} input ${X}$, label vector $Y$.\\
  \textbf{Output:} parameters in layer $i$ ${W_i}$,${B_i}$, prediction result $Q$.

   \begin{algorithmic}[1]
   \Function {$NeuronUpdate$}{$u,I$}
   \For {$t = 1$ \textbf{to} $T$}
    \State $u [ t ] = \tau_{decay}u[t-1]+I[t]$
   \EndFor
    \If {$u[t]<V_{th}$}
        \State $o[t] = 0$

    \Else
        \State $o[t] = 1$
        \State $u[t] = 0$

    \EndIf
    \State \Return{$o$}
   \EndFunction
      \end{algorithmic}
  \textbf{In the Training:}
  \begin{algorithmic}[1]
    \State \textbf{Forward:}
    \For {$t = 1$ \textbf{to} $T$}
      \State $x^{1}[t] = W_{1} \circledast X[t]+B_1$
    \EndFor
    \State $y^{1}\leftarrow tdBN\ for\ training(x^1)$
    \\$\verb|\\| \textbf{The\  training\  phase\ of\ Algorithm\ 1}$

    \State $o^{1}\leftarrow NeuronUpdate(u^{1},y^{1})$
    \For {$i = 2$ \textbf{to} $N$}
      \For {$t = 1$ \textbf{to} $T$}
        \State $x^{i}[t] = W_{i} \circledast o^{i-1}[t]+B_i$
      \EndFor
      \State $y^{i}\leftarrow tdBN\ for\ training(x^{i})$
      \State $o^{i}\leftarrow NeuronUpdate(u^{i},y^{i})$
    \EndFor
    \State $Q\leftarrow Decodinglayer(o^N)$
    \State $L\leftarrow ComputeLoss(Y,Q)$
    \State \textbf{Backward:}
    \State $\frac{\partial L}{\partial o^{i}}$,$ \frac{\partial L}{\partial u^{i}}\leftarrow Autograd$
  \end{algorithmic}
  \textbf{In the Inference:}
  \begin{algorithmic}[1]

      \State $x^{1}\leftarrow tdBN\ for\ inference(W_1,B_1,X)$
      \State $u^{1},o^{1}\leftarrow NeuronUpdate(u^{1},x^{1})$
      \For {$i = 2$ \textbf{to} $N$}

        \State $x^{i}\leftarrow tdBN\ for\ inference(W_i,B_i,x^{i-1})$
        \State $\verb|\\| \textbf{The\  inference\  phase\ of\ Algorithm\ 1}$

        \State $o^{i}\leftarrow NeuronUpdate(u^{i},x^{i})$
      \EndFor

    \State $Q\leftarrow Decodinglayer(o^N)$
  \end{algorithmic}
  \end{algorithm}
\section{B\ \ \ \ Proofs of Theorems}
\newtheorem{theorem4}{Theorem}
  \begin{theorem4}
  Consider an SNN with $T$ timesteps and the $j^{th}$ block's jacobian matrix at time $t$ is denoted as $J_j^t$. When $\tau_{decay}$ is equal to 0, if we fix the second moment of input vector and the output vector to $V_{th}^2$ for each block between two tdBN layers, we have $\phi(J_j^t(J_j^t)^T) \approx 1$ and the training of SNN can avoid gradient vanishing or explosion.
  \end{theorem4}
  \begin{proof}
  Consider an SNN with $T$ time step. When $\tau_{decay}$ is equal to 0, the membrane potential only depends on the input signal,so the gradient transforms independently in each timestep.

  At the timestep $t$, we consider the SNN as a series of blocks between two tdBN layers
  \begin{equation} f^t(x) = f^t_{i,\theta_i}\circ f^t_{i-1,\theta_i-1} \circ \cdots \circ f^t_{1,\theta_1}(x),\end{equation}
  where $f^t_{j,\theta_j}$ represents the function of $j^{th}$ block at timestep $t$ and define its input-output jacobian matrix as $\frac{\partial f^t_j}{\partial f^t_{j-1}}=J^t_j$.
  For each block, the second moment of input vector and the output vector are fixed to $V_{th}^2$, according to \textbf{Lemma 2}, we have
  \begin{equation} \phi(J_j^t(J_j^t)^T) = 1. \end{equation}
  So for each time step, according to \textbf{Lemma 1}, SNN with tdBN at timestep $t$ can achieve "\textbf{Block Dynamic Isometry}", which means $E[\|\Delta\theta_i^t\|_2^2]$ at any moment will not increase or diminish sharply when the SNN goes deep. And the $\Delta\theta_i$ can be defined by
  \begin{equation} \Delta\theta_i = \sum_{t=1}^T \Delta\theta_i^t.\end{equation}
  Under the condition that each $\Delta\theta_i^t$ at any timestep $t$ satisfies the $i.i.d$ assumption and  $E[\Delta\theta_i^t] \approx 0$, we have

  \begin{equation} E[(\Delta\theta_i)^2] \approx  \sum_{t=1}^T E[(\Delta\theta_i^t)^2]. \end{equation}
  So, the total $ E[\|\Delta\theta_i\|_2^2]$  will remain stable as well as $E[\|\Delta\theta_i^t\|_2^2]$, which means the SNN can avoid gradient vanishing or explosion.
  \end{proof}
  \newtheorem{theorem5}[theorem4]{Theorem}
  \begin{theorem5}
  With the iterative LIF model, assuming the pre-activation $x^t  \sim N(0,\sigma_{in}^2)$, we have the membrane potential $u^t \sim N(0,\sigma_{out}^2)$ and $\sigma_{out}^2 \propto \sigma_{in}^2$.
  \end{theorem5}
  \begin{proof}
   The iterative LIF model can be expressed by
   \begin{equation}   u^{t} = \tau_{decay}u^{t}(1-o^{t-1})+x^{t},\end{equation}\\
   \begin{equation}   o^{t} = \left\{ \begin{array}{ll} 1 & \textrm{if $u^{t}>V_{th}$}\\ 0 & \textrm{otherwise} \end{array} \right. ,\end{equation}
   where $u^{t}$ is the membrane potential of the neuron at timestep $t$, $o^{t}$ is the binary spike and $\tau_{decay}$ is the potential decay constant.
   So considering the membrane potential $u^t$ at the moment of $t$ and assuming its last firing time is time $t'$, we have
  \begin{equation} u^t = \sum_{p=t'}^{t} \tau_{decay}^{t-p}x^{p}, \end{equation}
  where $x^{p}$ denotes the pre-activation at $p$ moment and $\tau_{decay}$ represents the decay constant.
  Because $\tau_{decay}$ is a relative tiny constant(e.g. 0.25) in our SNN model, so
  \begin{equation} u^t \approx  \tau_{decay}x^{t-1}+x^t. \end{equation}
  Assume each $x^t$ is $i.i.d.$ sample from $N(0,\sigma_{in}^2)$, so the membrane potential $u^t \sim N(0,\sigma_{out}^2)$ and $\sigma_{out}^2 \propto \sigma_{in}^2$. With Fig. 3 in main paper, we find the high degree of similarity between the distributions of pre-activations and membrane potential, which supports the theorem.
  \end{proof}
\section{C\ \ \ \ Details of Experiments}
\subsection{Network architectures}
\begin{table*}[!htbp]
	\centering
	\caption{Network architectures of our ResNet-SNN}
    \label{table3}
    \begin{threeparttable}
	\begin{tabular}{c|c|c|c|c}
		\hline
		&17-layers&19-layers&34-layers(standard/large)&50-layers\\
		\hline
        \hline
        conv1&$3\times3,64,s1$&$3\times3,128,s1$&$7\times7,64/128,s2$&$7\times7,64,s2$\\
        \hline
        block1&$\left( \begin{array}{c}3\times3,64\\3\times3,64\end{array} \right)^*\times 3$&$\left( \begin{array}{c}3\times3,128\\3\times3,128\end{array}\right)\times3$&$\left( \begin{array}{c}3\times3,64/128\\3\times3,64/128\end{array} \right)^*\times 3$&$\left( \begin{array}{c}1\times1,64\\3\times3,64\\1\times1,256\end{array} \right)^*\times 3$\\
        \hline
        block2&$\left( \begin{array}{c}3\times3,128\\3\times3,128\end{array} \right)^*\times 4$&$\left( \begin{array}{c}3\times3,256\\3\times3,256\end{array} \right)^*\times 3$&$\left( \begin{array}{c}3\times3,128/256\\3\times3,128/256\end{array} \right)^*\times 4$&$\left( \begin{array}{c}1\times1,128\\3\times3,128\\1\times1,512\end{array} \right)^*\times 4$\\
        \hline
        block3&\diagbox[dir=SW,width=10em,height=4em]{}{}&$\left( \begin{array}{c}3\times3,512\\3\times3,512\end{array} \right)^*\times 2$&$\left( \begin{array}{c}3\times3,256/512\\3\times3,256/512\end{array} \right)^*\times 6$&$\left( \begin{array}{c}1\times1,256\\3\times3,256\\1\times1,1024\end{array} \right)^*\times 6$\\
        \hline
        block4&\diagbox[dir=SW,width=10em,height=4em]{}{}&\diagbox[dir=SW,width=10em,height=4em]{}{}&$\left( \begin{array}{c}3\times3,512/1024\\3\times3,512/1024\end{array} \right)^*\times 3$&$\left( \begin{array}{c}1\times1,512\\3\times3,512\\1\times1,2048\end{array} \right)^*\times 3$\\
        \hline
        &\multicolumn{2}{c}{average pool/2,256-d fc}&\multicolumn{2}{|c}{\diagbox[dir=SW,width=24em,height=1.1em]{}{}}\\
        \hline
        &11-d fc,softmax&10-d fc,softmax&\multicolumn{2}{c}{average pool,1000-d fc,softmax}\\

        \hline
	\end{tabular}
    \begin{tablenotes}
      \footnotesize
      \item * means the first basic block in the series perform downsampling directly with convolution kernels and a stride of 2.
    \end{tablenotes}
    \end{threeparttable}
\end{table*}
  In our experiment, we test ResNet-19 on CIFAR-10 and DVS-CIFAR10, ResNet-34 (standard/large) and ResNet-50 on ImageNet, ResNet-17 on DVS-Gessture. The detailed network architectures are shown in Table 4.
\subsection{Dataset introduction and pre-processing}
\subsubsection{CIFAR-10}
CIFAR-10 is an image dataset with 50000 training images and 10000 testing images with size of $32 \times 32$, which all belong to 10 classes. We randomly flip and crop every image. And every image is normalized by subtracting the global mean value of pixel intensity and devided by the standard variance along RGB channels during the data pre-processing.
\subsubsection{ImageNet}
ImageNet contains 1.28 million training images and 50000 validating images. Our data pre-processing uses the usual practice, which randomly crops and flips the $224 \times 224$ image with general normalization method as what we use in CIFAR-10.
\subsubsection{DVS-Gesture}
DVS-Gesture is a collection of moving gestures performed by 29 different individuals, which is captured by DVS cameras under three lighting conditions. It contains 1342 records from 23 subjects in training set and 288 examples from other 6 subjects in testing data. Belonging to 11 classes, each instance in DVS-Gesture is a stream of events with size of $128 \times 128$. In this experiment, we downsize the event steam to $32 \times 32$ and sample a slice every $30 ms$ during both training and testing process.  We set timestep $T$ to be 40. In each time step, the network receives only one slice, which means the first $1200 ms$ of each action is used for training or testing.
\subsubsection{DVS-CIFAR10}
 DVS-CIFAR10 is a neuromorphic dataset converted from famous CIFAR-10 to its dynamic form. It consists of 1000 images in the format of spike train per class with total 10 classes. In our experiment, the dataset is split into a training set with 9000 images and testing set with 1000 images. We downsample the original $128 \times 128$ image size to $42 \times 42$ reduce the temporal resolution by accumulating the spike train within every $5 ms$.
\subsection{Training settings}
\subsubsection{Optimizer}
For all of our experiments, we use the stochastic gradient descent (SGD) optimizer with initial learning rate $r = 0.1$ and momentum $0.9$. For ResNet-19 and ResNet-34 (standard/large), we let $r$ decay to $0.1r$ every 35 epoches. For ResNet-50, we let $r$ decay to $0.1r$ every 45 epoches. For ResNet-17, we let $r$ decay to $0.1r$ every 1000 epoches.
\subsubsection{Acceleration methods}
All of the models are programmed on Pytorch. And to accelerate the training process and reduce memory cost, we adopt the mixed precision training. Some models should be trained on multi-GPU, so we use the sync-BN technique to reduce the influence of small batch size. The batch sizes of every experiment are shown in Table 5.
\begin{table}[!htbp]
	\centering
	\caption{Batch size in experiments}

	\begin{tabular}{cccc}
		\toprule
		Network Architecture&Batch size per GPU&GPU \\
		\midrule
         ResNet-17&40&1\\
		 ResNet-19&36&1\\
         ResNet-34(standard)&24&8\\
         ResNet-34(large)&10&8\\
         ResNet-50&10&8\\
		\bottomrule
	\end{tabular}
\end{table}

\section{D\ \ \ \ Analysis of Computation Reduction}
\subsection {Computing patterns of SNNs and ANNs}
As we all know, SNN cannot beat ANN on accuracy comparisons of most recognition tasks. However, compared with ANNs, SNNs are much more energy-efficient due to their binary spikes and event-driven neural operations on the specialized hardware. For ANNs, the activations are determined by multiply-accumulate (MAC) operations. For SNNs, the compute cost is mainly the accumulate (AC) operations, which has two main differences: (1) In the fully-spiking SNN, there is no multiplication operation because of binary spike inputs. With Eq. (5) and (6), the pre-activations $x^t =\sum_{j}w_jo^{t}(j)$ and $o^{t}(j)=0\ or\ 1$, which means $x^t$ is an accumulation of weights $w_j$ in fact. So most compute operations we need in SNNs are additions not the multiplications. (2) The event-driven operation means that for each neuron the synaptic computation only occurs when required. If there is no spike received, there is no need to compute. With these characteristics, despite that SNNs need to be evaluated over $T$ timesteps, their computation costs are still lower than ANNs because the AC operations cost less than MAC operations and the sparse spikes decrease SNNs' computing operations a lot.
\subsection {Computation efficiency of SNNs}

In Fig. 6, we analyze the average numbers of spikes per neuron in each layer of ResNet-19 on CIFAR-10 and ResNet-34(large) on ImageNet. With much fewer timesteps, the spikes are surprisingly sparse. As the network goes deeper, the firing rates also decrease slowly. With these data, we can estimate how many AC operations our models need during a single inference process (ignore the computation of activation functions in ANNs and membrane potential update in SNNs). The results are shown in Table 6.
\begin{table}[!htp]
	\centering
	\caption{Computation comparisons between SNNs and ANNs}
	\begin{tabular}{c|c|c|c}
		\hline
        \hline
		&Models&Additions&Multiplications \\
		\hline
        \multirow{2}*{SNN}&ResNet-19&$1.8\times10^9$&$3.4\times10^7$\\
        \cline{2-4}
        &ResNet-34(large)&$1.2\times10^{10}$&$1.2\times10^9$\\
        \hline
        \multirow{2}*{ANN}&ResNet-19&$2.2\times10^9$&$2.2\times10^9$\\
        \cline{2-4}
        &ResNet-34(large)&$1.4\times10^{10}$&$1.4\times10^{10}$\\
        \hline
        \hline
	\end{tabular}
\end{table}

It should be pointed out that our SNN models need multiplications because the first encoding layer converts real images to spikes and the last decoding layer uses spiking signals to predict output classes. The comparison shows that sparse spikes in our SNN models lead to significant computation reduction and energy efficiency. Moreover, a single feed-forward pass in SNNs implemented in a neuromorphic architecture might be faster than an ANN implementation because of event-driven operations. So, with our much fewer timesteps, our SNNs can also show advantages on running time over ANNs of same architectures.

\begin{figure}[!htp]
\centering
\subfigure[Spiking activities in each layer of ResNet-19]{
\includegraphics[width=0.5\textwidth]{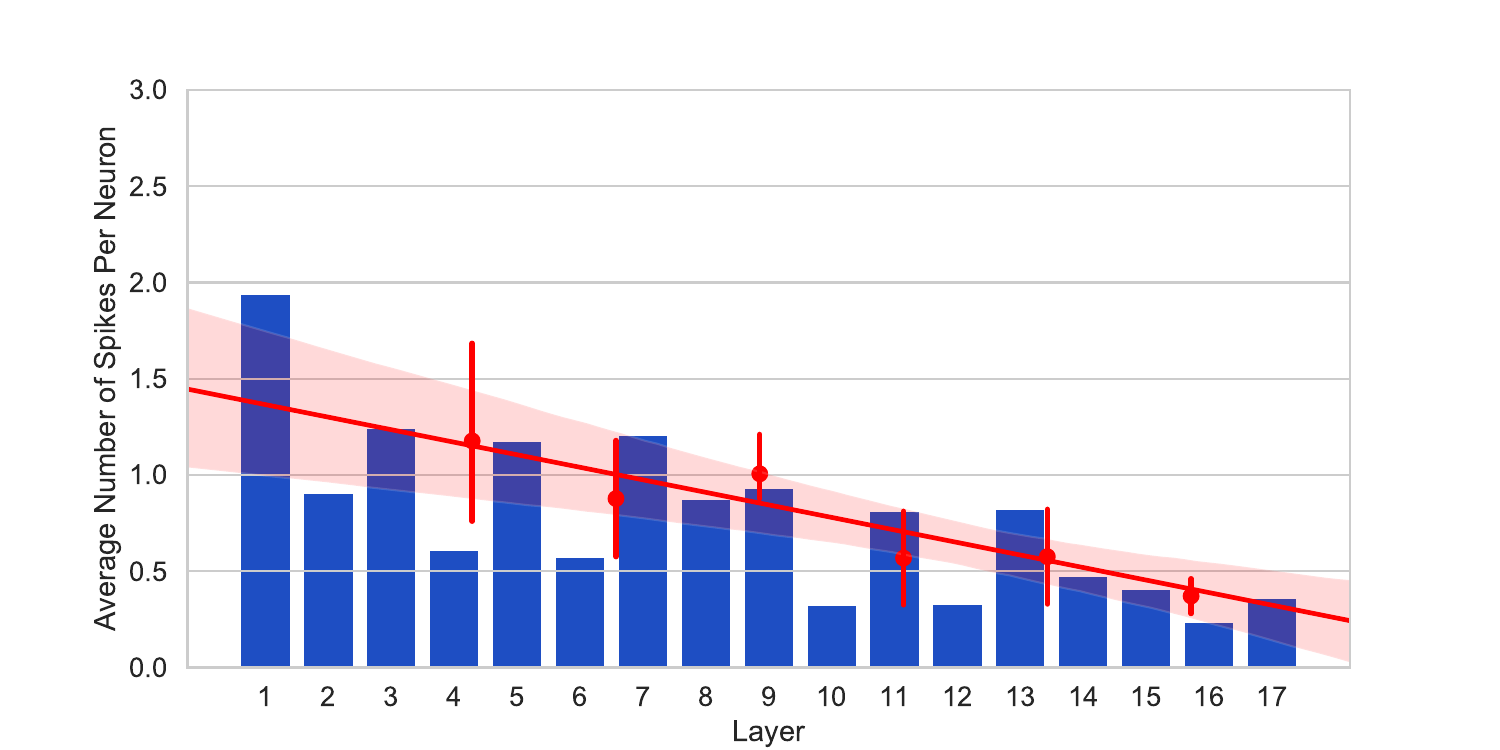}
}
\subfigure[Spiking activities in each layer of ResNet-34(large)]{
\includegraphics[width=0.5\textwidth]{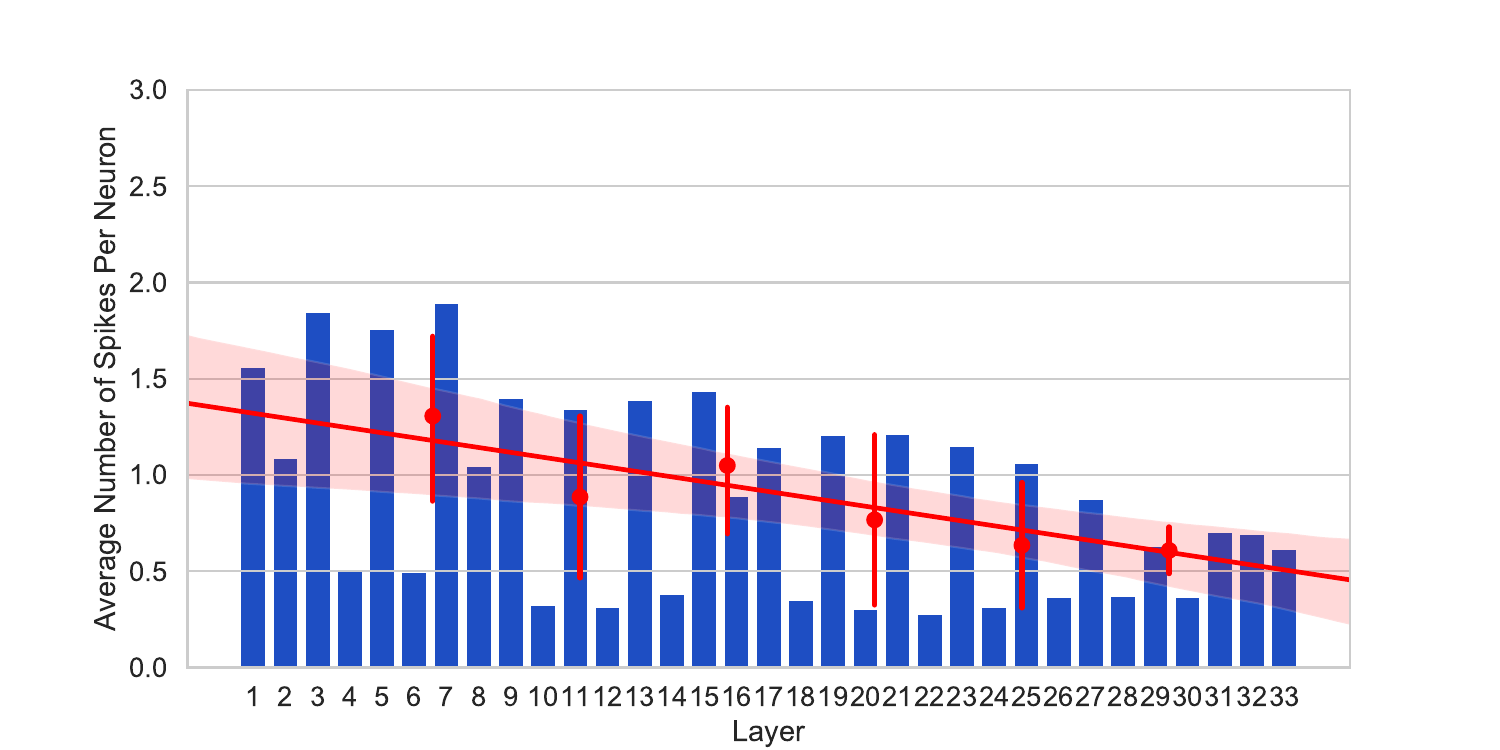}
}
\caption{\textbf{Spiking activities analysis of ResNet-19 and ResNet-34(large).}}
\end{figure}

 \end{document}